\RequirePackage[hyphens]{url}
\documentclass{article}

\pdfoutput=1

\PassOptionsToPackage{numbers}{natbib}
\PassOptionsToPackage{hyphens}{url}

\usepackage{url,hyperref}
\usepackage[numbers]{natbib}

\usepackage{amsfonts, amsmath, amsthm}
\usepackage{nicefrac}
\usepackage{mathtools}
\usepackage{enumitem}

\usepackage{setspace}
\usepackage[margin=1in]{geometry}
\newcommand\cond{\ | \ }
\newcommand{\eg}{e.g.}
\newcommand{\ie}{i.e.}
\newcommand\floor[1]{\left \lfloor #1 \right \rfloor}
\newcommand\ceil[1]{\left \lceil #1 \right \rceil}

\DeclareMathOperator{\argmax}{argmax}

\usepackage{apptools}
\AtAppendix{\counterwithin{theorem}{section}}

\newtheorem{theorem}{Theorem}
\newtheorem{proposition}{Proposition}
\newtheorem{lemma}{Lemma}
\newtheorem{corollary}{Corollary}

\theoremstyle{definition}
\newtheorem{assumption}{Assumption}

\hypersetup{
           breaklinks=true,
           colorlinks=false, 
           pdfusetitle=true,
           hypertexnames=false,
        }
        
\title{On the Convergence of Success Conditioning for Policy Optimization}
\author{Matthew Brun\thanks{Operations Research Center, Massachusetts Institute of Technology, Cambridge, MA 02139. Email: \href{mailto:brunm@mit.edu}{brunm@mit.edu}} \and Xu Andy Sun\thanks{Sloan School of Management, Massachusetts Institute of Technology, Cambridge, MA 02139.  Email: \href{mailto:sunx@mit.edu}{sunx@mit.edu}}}
\date{\vspace{-3em}}

\begin{document}

\maketitle

\begin{abstract}
Success conditioning is a strategy for improving decision-making policies in stochastic environments; it updates a policy by increasing the probability of taking actions that yield successful outcomes.  Success conditioning is common to many reinforcement learning applications, yet its limiting behavior and convergence rates are not well understood.  In this work, we demonstrate that success conditioning converges to an optimal policy on a broad class of Markov decision processes (MDPs).  We also derive convergence rates in some common settings.  For discounted MDPs, we prove convergence within $\mathcal{O}(1/\varepsilon^p)$ iterations to an $\varepsilon$-optimal policy, where the exponent $p$ depends on problem data.  For single-period MDPs, such a policy is obtained within $\mathcal{O}(\log(1/\varepsilon))$ iterations.
\end{abstract}

\section{Introduction}

Success conditioning (SC) is an update rule for policies in stochastic sequential decision-making processes.   SC improves a randomized policy by increasing the probability of taking actions that yield high rewards and decreasing the probability of taking those which are not fruitful.  When outcomes of the process are labeled as successes and failures, the policy update mimics the distribution of actions which yield successful outcomes.

SC and related strategies appear across modern artificial intelligence (AI) contexts, especially in machine learning and reinforcement learning \citep{andrychowicz2017hindsight,ding2019goal}.  \emph{Rejection sampling fine-tuning} refines generative models by sampling outputs, discarding those with low-quality, and continuing to train on the best responses \citep{dong2023raft}. Post-training pipelines for large language models apply this strategy \citep{grattafiori2024llama}, and similar concepts are embedded directly in the training process via reinforcement learning \citep{guo2025deepseek}.  Models for robot control have been trained on successful trajectories from other robots \citep{brohan2022rt} and can autonomously fine-tune on successful outcomes from practice sessions \citep{ghasemipour2025self}. In reinforcement learning, \emph{decision transformers} learn to predict actions as a function of downstream rewards, then identify high-quality actions by conditioning on large returns \citep{chen2021decision}.

The SC update was formalized by \citet{russo2026success} for a Markov decision process (MDP) where rewards are received at a terminal state.  Equivalent updates for dense rewards appear in work on outcome-conditioned behavioral cloning \citep{eysenbach2022imitating} and expectation-maximization algorithms for MDPs \citep{toussaint2011expectation,strupl2022reward}.  SC is closely related to \emph{reward conditioning}, which updates policies by conditioning the action distribution on a reward threshold \citep{peters2007reinforcement,brandfonbrener2022does}.  In previous work, SC has primarily been treated as a heuristic method for policy improvement.  Convergence of the SC update has not been rigorously derived and was recently posed as an open question by \citet{russo2026success}.

In this work, we treat SC as an iterative method for policy optimization and present its convergence properties.  We prove that, for a broad class of MDPs, SC converges to a restricted optimal policy: one that is optimal on the set of supported actions in the initial policy.  For single-period MDPs (\ie, multi-armed bandit problems) and discounted MDPs, we derive upper bounds on the convergence rates to $\varepsilon$-optimal solutions. SC converges in no more than $\mathcal{O}(\log(1/\varepsilon))$ iterations in the single-period case, and in no more than $\mathcal{O}(1/\varepsilon^p)$ iterations in the discounted case, where the exponent $p$ depends on the instance.  These results address several open questions on the limiting behavior and convergence rates of SC. 

\citet{strupl2022reward} analyze reward-weighted regression (RWR), an expectation-maximization algorithm for policy optimization in MDPs.  This analysis demonstrates that RWR and SC have equivalent updates and proves that RWR converges in the limit to an optimal policy, under similar assumptions to those described in our work.  \citet{strupl2022reward} also show that, after sufficiently many iterations, the optimality gap of RWR is bounded by a geometric series, which implies asymptotic linear convergence for discounted MDPs.  However, it remains unknown how many iterations are required before this linear convergence rate applies.  In contrast, we prove explicit iteration complexities for SC, parametrized by the initial policy.  An extensive comparison with \citet{strupl2022reward} is provided in Appendix~\ref{app:RLRcompare}.

\section{Problem Description}

Consider a Markov decision process (MDP), defined by a finite set of states $\mathcal{S}$, a finite set of actions $\mathcal{A}$, transition probabilities $P$, and rewards $R$. The reward $R(s,a)$ associated with state $s \in \mathcal{S}$ and action $a \in \mathcal{A}$ without loss of generality falls on $[0,1]$.  The probability of a transition to state $s' \in \mathcal{S}$ from state $s \in \mathcal{S}$ when action $a \in \mathcal{A}$ is taken is given by $P(s' \cond s,a)$.  Actions are chosen according to a stationary and memoryless policy $\pi$ that depends only on the current state $s$, where $\pi(a \cond s)$ gives the probability of taking action $a \in \mathcal{A}$ when in state $s \in \mathcal{S}$.  Let $\Pi$ denote the set of such policies.

A \emph{trajectory} of the MDP is an infinite sequence of states and actions: $(s_0,a_0,s_1,a_1,\ldots)$, where $s_t \in \mathcal{S}$ and $a_t \in \mathcal{A}$.  Transitions between states are distributed as $s_{t+1} \sim P(\cdot \cond s_t,a_t)$ and actions as $a_t \sim \pi(\cdot \cond s_t)$.  We select an absorbing state $\overline{s}$, with $R(\overline{s},a) = 0$ and $P(\overline{s} \cond \overline{s},a) = 1$ for all $a \in \mathcal{A}$.  This state represents termination of the process; that is, once $\overline{s}$ is entered, no other state is reached and no additional rewards are obtained. The set of non-absorbing states is given by $\mathcal{S}' := \mathcal{S} \setminus \{\overline{s}\}$.

We seek to identify policies $\pi \in \Pi$ that maximize the expected cumulative reward:
\begin{equation}
    \label{eq:optimizePolicy}
    \max_{\pi \in \Pi} \ \mathbb{E}_{\pi} \left [ \sum_{t = 0}^\infty \gamma^t R(s_t,a_t)\right ],
\end{equation}
where $\gamma \in [0,1]$ is the discount factor.  Undiscounted reward functions ($\gamma = 1$) are permitted.
We define a \emph{value function} under policy $\pi$, which gives the expected cumulative reward from a trajectory that starts at state $s$:
\begin{equation}
    V_{\pi}(s) := \mathbb{E}_{\pi} \left [ \sum_{t=0}^\infty \gamma^t R(s_t,a_t) \ \middle| \ s_0 = s \right ] \quad \forall s \in \mathcal{S}.
\end{equation}
We also define \emph{Q-functions}, which give the expected cumulative reward of taking action $a$ at state $s$:
\begin{equation}
    Q_{\pi}(s,a) := \mathbb{E}_{\pi} \left [\sum_{t=0}^\infty \gamma^t R(s_t,a_t) \ \middle| \ s_0 = s,\ a_0 = a \right ] \quad \forall a \in \mathcal{A},\ s \in \mathcal{S}.
\end{equation}

The value functions are a solution to the \emph{Bellman equations}:
\begin{equation}    
    \label{eq:bellmanValue}
    \begin{aligned}
        V_{\pi}(s) & = \sum_{a \in \mathcal{A}} \pi(a \cond s) \left ( R(s,a) + \gamma \sum_{s' \in \mathcal{S}} P(s' \cond s,a) V_\pi(s') \right ) \quad \forall s \in \mathcal{S}',
    \end{aligned}
\end{equation}
where $V_{\pi}(\overline{s}) = 0$.  From this system, we can relate the value and Q-functions as follows:
\begin{equation}
    \label{eq:valueQRelation}
    V_{\pi}(s) = \sum_{a \in \mathcal{A}} \pi(a \cond s) Q_{\pi}(s,a) \quad \text{and} \quad Q_{\pi}(s,a) = R(s,a) + \gamma\sum_{s' \in \mathcal{S}} P(s' \cond s,a) V_{\pi}(s').
\end{equation}
For some instances, the Bellman equations~\eqref{eq:bellmanValue} may have multiple solutions.  To eliminate this and other aberrant behavior, we make the following assumption.  Here, $\pi_0$ denotes an incumbent policy, which will be the initial point for the SC algorithm.

\begin{assumption}
    \label{assumption:model}
    The following properties hold:
    \begin{enumerate}
        \item For any $\pi \in \Pi$, the value functions $V_{\pi}(s)$ are finite and the unique solution to~\eqref{eq:bellmanValue}; and 
        \item $V_{\pi_0}(s) > 0$ for all $s \in \mathcal{S}'$.
    \end{enumerate}
\end{assumption}

Assumption~\ref{assumption:model} ensures that the MDP is reasonably well-behaved.  First, it requires that the value functions under any policy are the unique solution to the Bellman equations.  This assumption eliminates ill-behaved problems, \eg, undiscounted instances that cycle infinitely between states while accruing rewards.  The condition is satisfied for discounted problems (i.e., $\gamma < 1$) or for undiscounted stochastic shortest path problems that satisfy minor assumptions, for example, that the absorbing state $\overline{s}$ can be reached with nonzero probability from any state under any policy \citep[][Sec.~4.2]{bertsekas2019reinforcement}.  Stochastic shortest path and discounted problems are common classes for analysis of infinite-horizon MDPs.  Second, the assumption requires that the value of all states (except $\overline{s}$) must be strictly positive under the initial policy $\pi_0$.  Although this property may not hold in general, states $s' \in \mathcal{S}'$ with $V_{\pi_0}(s') = 0$ can be removed from the MDP and transitions into $s'$ reassigned to the absorbing state $\overline{s}$.  This construction generates an MDP that satisfies the latter assumption but prohibits optimization of the policy at states with zero initial value.  

\subsection{The Success Conditioning Update}

SC iteratively applies the \emph{success conditioning update} to the incumbent policy.  Let $\pi_k$ be the policy in the $k$-th iteration.  For convenience, we index the value and Q-functions by the iteration index: $Q_k(s,a) := Q_{\pi_k}(s,a)$ and $V_k(s) := V_{\pi_k}(s)$.  The policy in the $(k+1)$-th iteration is computed by
\begin{equation}
    \label{eq:successConditioning}
    \pi_{k + 1}(a \cond s) = \frac{\pi_{k}(a \cond s) Q_{k}(s,a)}{V_{k}(s)} \quad \forall a \in \mathcal{A},\ s \in \mathcal{S}'.
\end{equation}
The initial policy $\pi_0$ is given.  This update increases the probability of selecting actions whose expected rewards are above average under the current policy, and decreases the probability of actions with below average expected rewards.
For the update to be well-defined, $V_{k}(s)$ must be positive.  In Lemma~\ref{lemma:monotonic}, we show that the value and Q-functions increase under this update; positivity of $V_k(s)$ follows from Assumption~\ref{assumption:model}.  Lemma~\ref{lemma:monotonic} generalizes Corollary~4.5 of \citet{russo2026success} to the dense reward setting.  The proofs of Lemma~\ref{lemma:monotonic} and all other lemmas are deferred to Appendix~\ref{app:proofs}.
\begin{lemma}
    \label{lemma:monotonic}
    Let Assumption~\ref{assumption:model} hold. Then, $V_{k+1}(s) \geq V_{k}(s)$ and $Q_{k+1}(s,a) \geq Q_{k}(s,a)$ for all $s \in \mathcal{S}'$, $a \in \mathcal{A}$, and $k \in \mathbb{Z}_{\geq 0}$
\end{lemma}

Under the update~\eqref{eq:successConditioning}, the policy $\pi_k$ and value function $V_k$ at iteration $k$ can be written in closed form:
\begin{equation}
    \label{eq:SCClosedForm}
    \pi_k(a \cond s) = \frac{\pi_0(a \cond s) \prod_{j=0}^{k-1} Q_{j}(s,a)}{\prod_{j=0}^{k-1} V_{j}(s)} \quad \text{and} \quad V_{k}(s) = \sum_{a \in \mathcal{A}} \frac{\pi_0(a \cond s) \prod_{j=0}^{k} Q_{j}(s,a)}{\prod_{j=0}^{k-1} V_{j}(s)}.
\end{equation}

\subsection{Motivation by Conditional Probability under Binary Rewards}

The SC update is motivated by a setting with binary rewards that are received at the termination of an undiscounted MDP \citep{russo2026success}.  This setting is a restriction of the one presented above, in which: a) there exists a set of terminal states $\mathcal{T} \subset \mathcal{S}'$ that transition directly into the absorbing state ($P(\overline{s} \cond \tau,a) = 1$ for all $a \in \mathcal{A}$ and $\tau \in \mathcal{T}$); b) rewards are only received in terminal states ($R(s,a) = 0$ for $s \not \in \mathcal{T}$), are independent of the action ($R(\tau) = R(\tau,a)$ for all $\tau \in \mathcal{T}$ and $a \in \mathcal{A}$), and are binary ($R(\tau) \in \{0,1\}$ for $\tau \in \mathcal{T}$); and c) the MDP is undiscounted ($\gamma = 1$).

Let $\tau \in \mathcal{T}$ be a random variable that gives the terminal state of the trajectory.  The SC update is 
\begin{equation}
    \begin{aligned}
        \pi_{k+1}(a \cond s) &  = \frac{\pi_{k}(a \cond s) Q_{k}(s,a)}{V_{k}(s)} = \frac{\mathbb{P}_{\pi_k} \left[a_0 = a \cond s_0 = s \right ] \mathbb{P}_{\pi_k} \left[R(\tau) = 1 \cond a_0 = a,\ s_0 = s \right ]}{\mathbb{P}_{\pi_k} \left[ R(\tau) = 1 \cond s_0 = s\right ]}\\
        & = \mathbb{P}_{\pi_k} \left[ a_0 = a \cond s_0 = s,\ R(\tau) = 1 \right ],
    \end{aligned}
\end{equation}
where the first line replaces the policy, value function, and Q-function with their probabilistic interpretations in the undiscounted setting with binary rewards, and the second line is an application of Bayes' rule.  Thus, the SC update generates a new policy that corresponds to the distribution of actions that yield \emph{successful} trajectories under policy $\pi_k$, where successful trajectories are those that receive a reward of $1$.  In this work, we study the behavior of SC for general MDPs; however, this general setting does not permit interpretation of the next policy iterate as a conditional distribution.

\subsection{Optimal Restricted Value Functions}

For each $s \in \mathcal{S}'$, let $\mathcal{A}_0(s):= \{a \in \mathcal{A} \,:\, \pi_0(a \cond s) > 0\}$ be the set of supported actions under the initial policy.  Clearly, $\pi_k(a \cond s) = 0$ for any $a \not \in \mathcal{A}_0(s)$ by~\eqref{eq:SCClosedForm}.  Let $\Pi' \subset \Pi$ represent the set of policies $\pi$ which satisfy $\pi(a \cond s) = 0$ for all $s \in \mathcal{S}'$ and $a \in \mathcal{A} \setminus \mathcal{A}_0(s)$.  We define $V'_{k}(s) := \max_{a \in \mathcal{A}_0(s)} Q_{k}(s,a)$, which gives the value of the best supported action at state $s$, when policy $\pi_k$ is followed downstream of the current state.  The value functions $V'_k$ are a greedy value iteration step from the incumbent SC value function $V_k$.

Let the \emph{optimal restricted value functions} over the policy set $\Pi'$ be given by
\begin{equation}
    V^*(s) := \max_{\pi \in \Pi'} \ \mathbb{E}_{\pi} \left [ \sum_{t=0}^\infty \gamma^t R(s_t,a_t) \ \middle| \ s_0 = s \right ] \quad \forall s \in \mathcal{S}.
\end{equation}
Under Assumption~\ref{assumption:model}, $V^*$ are the unique solutions to the \emph{restricted} Bellman optimality equations:
\begin{equation}
    \label{eq:bellmanOptimality}
        V^*(s) = \max_{a \in \mathcal{A}_0(s)} \ R(s,a) + \gamma \sum_{s' \in \mathcal{S}} P(s' \cond s,a) V^*(s') \quad \forall s \in \mathcal{S}',
\end{equation}
where $V^*(\overline{s}) = 0$. The optimal restricted Q-functions $Q^*$ are defined analagously.   In this work, we show convergence of the SC value functions $V_{k}$ to the optimal restricted value functions $V^*$.  

\section{Limiting Behavior of Success Conditioning}

In this section, we demonstrate that SC converges in the limit to the optimal restricted value functions~$V^*$.  We begin by analyzing the ratio of the value functions $V_{k}$ to the value iteration step $V'_k$.  Proposition~\ref{proposition:limitMeanConvergence} shows that the average of the ratios converges to one as the iteration count increases.
\begin{proposition}
    \label{proposition:limitMeanConvergence}
    Let Assumption~\ref{assumption:model} hold.  For all $s \in \mathcal{S}'$, 
    \begin{equation*}
        \lim_{k \rightarrow \infty}\ \frac{1}{k+1} \sum_{j=0}^k \frac{V_{j}(s)}{V'_{j}(s)} \ =\ 1.
    \end{equation*}
\end{proposition}

\begin{proof}
    By Lemma~\ref{lemma:monotonic}, $V_{k}(s)$ and $Q_{k}(s,a)$ are increasing sequences in $k$ that are bounded above by $V^*(s)$ and $Q^*(s,a)$.  Therefore, the sequences have limits, which we define for $s \in \mathcal{S}'$ and $a \in \mathcal{A}$ as
    \begin{equation*}
        \overline{V}(s) := \lim_{k \rightarrow \infty} V_{k}(s) \quad \text{and} \quad \overline{Q}(s,a) := \lim_{k \rightarrow \infty} Q_{k}(s,a).
    \end{equation*}
    For any $\varepsilon \in (0,1)$, there is some $\overline{k}_\varepsilon$ such that $(1-\varepsilon) \overline{Q}(s,a) \leq Q_{k}(s,a)$ for all $k \geq \overline{k}_\varepsilon$, $s \in \mathcal{S}'$, and $a \in \mathcal{A}$.  Let $a^*(s) \in \argmax_{a \in \mathcal{A}_0(s)} \overline{Q}(s,a)$ for all $s \in \mathcal{S}'$.  For all $k \in \mathbb{Z}_{\geq 0}$ and $s \in \mathcal{S}'$,
    \begin{equation}
        \label{eq:limitMaxBound}
        V'_k(s) = \max_{a \in \mathcal{A}_0(s)} Q_k(s,a) \leq \max_{a \in \mathcal{A}_0(s)} \overline{Q}(s,a) = \overline{Q}(s,a^*(s)).
    \end{equation}

    Fix some $\varepsilon \in (0,1)$ and let $k \geq \overline{k}_\varepsilon$.  From the closed-form representation of the value function~\eqref{eq:SCClosedForm},
    \begin{equation}
        \label{eq:limitRatioBound}
        \begin{aligned}
            \prod_{j=0}^k \frac{V_{j}(s)}{V'_{j}(s)} & \geq \pi_0(a^*(s) \cond s) \prod_{j=0}^k \frac{Q_{j}(s,a^*(s))}{V'_{j}(s)}\\
            & \geq \pi_0(a^*(s) \cond s) \left (\frac{Q_{0}(s,a^*(s))}{V^*(s)} \right ) ^{\overline{k}_\varepsilon} \prod_{j=\overline{k}_\varepsilon}^k \frac{Q_{j}(s,a^*(s))}{\overline{Q}(s,a^*(s))}\\
            & \geq \pi_0(a^*(s) \cond s) \left (\frac{Q_{0}(s,a^*(s))}{V^*(s)} \right ) ^{\overline{k}_\varepsilon} (1 - \varepsilon)^{k-\overline{k}_\varepsilon+1}.
        \end{aligned}
    \end{equation}
    The first inequality selects a single element from the sum in \eqref{eq:SCClosedForm}, the second uses~\eqref{eq:limitMaxBound} and the monotonicity of $Q_{k}$ (Lemma~\ref{lemma:monotonic}), and the third uses the definition of $\overline{k}_\varepsilon$.  

    As the sequence $V_k(s)$ has a limit, the sequence $V'_k(s)$ also has a limit, which is bounded away from zero as $V'_k(s) \geq V_0(s) > 0$ by Assumption~\ref{assumption:model}.  Therefore, the limit of the ratios $V_k(s) / V'_k(s)$ also exists, which is sufficient to ensure the existence of the limits of the partial arithmetic and geometric means in the following inequalities:
    \begin{equation*}
        \begin{aligned}
            1 & \geq \lim_{k \rightarrow \infty} \sum_{j=0}^k \frac{V_{j}(s)}{(k+1) V'_{j}(s)} \geq \lim_{k \rightarrow \infty} \left ( \prod_{j=0}^k \frac{V_{j}(s)}{V'_{j}(s)} \right ) ^{\frac{1}{k+1}}\\
            & \geq \sup_{\varepsilon \in (0,1)} \ \lim_{k \rightarrow \infty} \left (\pi_0(a^*(s) \cond s) \left (\frac{Q_{0}(s,a^*(s))}{V^*(s)} \right ) ^{\overline{k}_\varepsilon} (1 - \varepsilon)^{k-\overline{k}_\varepsilon+1} \right )^{\frac{1}{k+1}} = \sup_{\varepsilon \in  (0,1)} \ (1 - \varepsilon) = 1.
        \end{aligned}
    \end{equation*}
    Here, the first relation is by definition of $V'_j(s)$, the second is from the AM-GM inequality, the third uses~\eqref{eq:limitRatioBound}, and the fourth is due to the evaluation of the limit for fixed $\varepsilon$.
\end{proof}

From Proposition~\ref{proposition:limitMeanConvergence}, we establish convergence of $V_{k}$ to the optimal restricted value functions $V^*$.  Theorem~\ref{theorem:limitConvergence} presents this result.

\begin{theorem}
    \label{theorem:limitConvergence}
    Let Assumption~\ref{assumption:model} hold.  For all $s \in \mathcal{S}'$,
    \begin{equation*}
        \lim_{k \rightarrow \infty} \ V_{k}(s) = V^*(s);
    \end{equation*}
    that is, the value functions of the SC policies converge to the restricted optimal value functions.
\end{theorem}

\begin{proof}
    As noted in the proof of Proposition~\ref{proposition:limitMeanConvergence}, the sequences $V_k(s)$, $V'_k(s)$, and $V_k(s) / V'_k(s)$ all have limits.  We define
    \begin{equation*}
        L := \frac{\lim_{k \rightarrow \infty} V_{k}(s)}{\lim_{k \rightarrow \infty} V'_{k}(s)} = \lim_{k \rightarrow \infty} \frac{V_{k}(s)}{V'_{{k}}(s)} = \lim_{k \rightarrow \infty}\ \frac{1}{k+1} \sum_{j = 0}^k \frac{V_{j}(s)}{V'_{j}(s)}.
    \end{equation*}
    The first equality follows from the existence of both limits and the positivity of the denominator under Assumption~\ref{assumption:model} ($V'_k(s) \geq V_0(s) > 0$), and the second from the equality of the limit of a sequence and the limit of its averages.
    By Proposition~\ref{proposition:limitMeanConvergence}, we have that $L = 1$, so
    \begin{equation*}
        \lim_{k \rightarrow \infty} V_{k}(s) = \lim_{k \rightarrow \infty} V'_{k}(s) = \max_{a \in \mathcal{A}_0(s)}\ R(s,a) + \gamma\sum_{s' \in \mathcal{S}} P(s' \cond s,a) \lim_{k \rightarrow \infty} V_{k}(s).
    \end{equation*}
    Therefore, the values $\lim_{k \rightarrow \infty} V_{k}(s)$ satisfy the Bellman equations~\eqref{eq:bellmanOptimality} and, under Assumption~\ref{assumption:model}, are the unique solution $V^*(s)$.
\end{proof}

Theorem~\ref{theorem:limitConvergence} establishes the convergence behavior of SC: policies generated by SC converge in value to the optimal value functions over actions that are supported by the initial policy $\pi_0$.  In the following sections, we derive convergence rates for MDPs with specific structures.

\section{Convergence Rate for Single-Period MDPs}
\label{sec:singlePeriod}

In this section, we consider an MDP that consists of two states ($\mathcal{S} = \{s^\circ, \overline{s}\}$) and reaches the absorbing state $\overline{s}$ immediately after the first action ($P(\overline{s} \cond s^\circ,a) = 1$ for all $a \in \mathcal{A}$).  Under this structure, $Q_{\pi}(s^\circ,a) = R(s^\circ,a)$.  As the action taken in the first period entirely determines the infinite-horizon reward, we refer to this problem as a \emph{single-period} MDP.  Single-period MDPs are multi-armed bandit problems in which the rewards associated with each arm are known exactly, and are equivalent to maximally-discounted MDPs, \ie, those with $\gamma = 0$.  Although single-period MDPs can be solved in closed-form by selecting an action which maximizes $R(s^\circ,a)$, they provide a simple and instructive setting for analysis of SC convergence rates.  With a similar perspective, \citet{russo2026success} proposes an example of a single-period MDP that demonstrates how the SC update may make slow progress towards an optimal policy. 

For single-period MDPs, we derive a rate of convergence of the value function $V_{k}(s^\circ)$ to the optimal restricted value function $V^*(s^\circ) = \max_{a \in \mathcal{A}_0(s^\circ)} R(s^\circ,a)$.  Theorem~\ref{theorem:MABConvergenceRate} presents the convergence rate.  This rate shows that a policy with an optimality gap of $\varepsilon$ is found within $\mathcal{O}(\log(1/\varepsilon))$ iterations.  However, the coefficient on this rate becomes arbitrarily large as the value of the second-best action approaches that of the best action.  To avoid this issue, we also show that such a policy is found within $\mathcal{O}(1/\varepsilon)$ iterations, with a coefficient that depends only on the initial policy.

\begin{theorem}
    \label{theorem:MABConvergenceRate}
    Let Assumption~\ref{assumption:model} hold.  Let $\mathcal{A}^*_0 := \{a \in \mathcal{A}_0(s^\circ)\,:\, R(s^\circ,a) = V^*(s^\circ)\}$ be the set of optimal actions, let $R_2 := \max_{a \in \mathcal{A}_0(s^\circ) \setminus \mathcal{A}^*_0}\, R(s^\circ,a)$ be the reward of the second-best action, let $D := \sum_{a \in \mathcal{A}^*_0} \pi_0(a \cond s^\circ)$, and let $E := V^*(s^\circ)/R_2$.  For $\varepsilon \in (0,1)$ and $k \in \mathbb{Z}_{\geq 0}$ such that
    \begin{equation}
    \label{eq:IterationConditionMAB}
        k \geq \min \left \{\frac{1}{\log (E)} \log \left(\frac{1-D}{D \varepsilon} \right ),\  \frac{ \log\left (1/D \right )}{\varepsilon} \right \},
    \end{equation}
    it holds that $V^*(s^\circ) - V_k(s^\circ) \leq \varepsilon$.
\end{theorem}

\begin{proof}
    Suppose that $k$ satisfies \eqref{eq:IterationConditionMAB}.  We first explore the case where $k \geq \log(1/D)/\varepsilon$.  We have that
    \begin{equation}
        \label{eq:singlePeriodRatioBound}
        \begin{aligned}
            \left (\frac{V_{k}(s^\circ)}{V^*(s^\circ)} \right )^{k+1} & \geq \prod_{j=0}^k \frac{V_{j}(s^\circ)}{V^*(s^\circ)} = \sum_{a \in \mathcal{A}} \pi_0(a \cond s^\circ) \left ( \frac{R(s^\circ,a)}{V^*(s^\circ)} \right )^{k+1}\\
            & \geq \sum_{a \in \mathcal{A}^*_0} \pi_0(a \cond s^\circ) \left ( \frac{R(s^\circ,a)}{V^*(s^\circ)} \right )^{k+1} = \sum_{a \in \mathcal{A}^*_0} \pi_0(a \cond s^\circ) = D,
        \end{aligned}
    \end{equation}
    where the first inequality follows from Lemma~\ref{lemma:monotonic}, the first equality from the closed-form description~\eqref{eq:SCClosedForm} of the value functions, the second inequality by selecting elements from the sum, and the second equality from the definition of $\mathcal{A}^*_0$.  Then, 
    \begin{equation*}
        V^*(s^\circ) - V_k(s^\circ) = V^*(s^\circ) \left ( 1 - \frac{V_k(s^\circ)}{V^*(s^\circ)} \right ) \leq 1 - D^{\frac{1}{k}} \leq 1 - \exp(-\varepsilon) \leq \varepsilon,
    \end{equation*}
    where the first inequality follows by taking the $(k+1)$-th root of~\eqref{eq:singlePeriodRatioBound} and from $R(s^\circ,a) \leq 1$, and the final inequality is due to the relation $1+x \leq \exp(x)$ for all $x\in\mathbb{R}$.
    
    Second, we explore the case where $k \geq \log((1-D)/D\varepsilon)/\log(E)$.  From~\eqref{eq:SCClosedForm}, we have that
    \begin{equation}
        \label{eq:alternateClosedForm}
        V_k(s^\circ) = \frac{\sum_{a \in \mathcal{A}} \pi_0(a \cond s^\circ) R(s^\circ,a)^{k+1}}{\sum_{a \in \mathcal{A}} \pi_0(a \cond s^\circ) R(s^\circ,a)^k}.
    \end{equation}
    Then,
    \begin{equation*}
        \begin{aligned}
            V^*(s^\circ) - V_k(s^\circ) & = \frac{\sum_{a \in \mathcal{A}} \pi_0(a \cond s^\circ) \left (V^*(s^\circ) - R(s^\circ,a) \right ) R(s^\circ,a)^k}{\sum_{a \in \mathcal{A}} \pi_0(a \cond s^\circ) R(s^\circ,a)^k}\\
            & \leq \frac{\sum_{a \in \mathcal{A}_0(s^\circ) \setminus A^*_0} \pi_0(a \cond s^\circ) \left (V^*(s^\circ) - R(s^\circ,a) \right ) R(s^\circ,a)^k}{\sum_{a \in \mathcal{A}^*_0} \pi_0(a \cond s^\circ) V^*(s^\circ)^k} \leq \frac{1-D}{D E^k} \leq \varepsilon,
        \end{aligned}
    \end{equation*}
    where the first relation follows from~\eqref{eq:alternateClosedForm}, the second relation follows from the definitions of $\mathcal{A}_0(s^\circ)$ and $\mathcal{A}^*_0$, and the third relation is due to $R(s^\circ,a) \in [0,1]$ and the definition of $E$.
\end{proof}

\section{Convergence Rate for Discounted MDPs}

In this section, we derive a convergence rate for discounted MDPs, that is, problems with $\gamma < 1$.  Discounting is common for infinite-horizon MDPs, and standard convergence rates for many typical policy optimization algorithms assume that the MDP is discounted \citep{puterman2014markov}.  The analysis in this section requires Assumption~\ref{assumption:restrictedDiscount}, which requires discounting and that the value functions $V^*$ and $V_k$ fall on $(0,1]$.  The latter assumption is for convenience and can be achieved by rescaling the rewards $R(s,a)$.
\begin{assumption}
    \label{assumption:restrictedDiscount}
    The following properties hold:
    \begin{enumerate}
        \item The MDP is discounted with $\gamma \in [0,1)$; and
        \item $V^*(s) \in (0,1]$ for all $s \in \mathcal{S}'$.
    \end{enumerate}
\end{assumption}
We also define
\begin{equation*}
\begin{aligned}
    \underline{\pi} := \min_{\substack{s \in \mathcal{S}'\\a \in \mathcal{A}_0(s)}}\ \pi_0(a \cond s), \quad \Phi := \min_{s \in \mathcal{S}'} \ \max_{a \in \mathcal{A}_0(s)} \ Q_0(s,a),\quad \text{and} \quad a^*_k(s) \in \underset{a \in \mathcal{A}_0(s)}{\argmax} \ \ Q_{k}(s,a).
\end{aligned}
\end{equation*}
Parameter $\underline{\pi}$ gives the smallest weight on a supported action under the initial policy, $\Phi$ is a lower bound across states on the value of the best action under the initial policy, and $a^*_k(s)$ gives the best action at state $s$ under policy $\pi_k$.  By Assumption~\ref{assumption:model} and the definition of $\mathcal{A}_0(s)$, both $\underline{\pi}$ and $\Phi$ are strictly positive.

To derive a convergence rate, we define a sequence of indices $\{j_i\}_{i=0}^\infty \subset \mathbb{Z}_{\geq 0}$ and treat the SC value functions $V_{j_i}$ as iterates of an approximate value iteration (AVI) algorithm, where there is some error $\varepsilon_i$ between the value iteration step $V'_{j_i}$ and the SC value function $V_{j_{i+1}}$ at the next index.  Lemma~\ref{lemma:AVIconvergence} gives a convergence rate for AVI, parametrized by the errors $\varepsilon_i$.  Then, Proposition~\ref{proposition:convergencetoAVI} bounds the errors $\varepsilon_i$ by the values of the index sequence $\{j_i\}_{i=0}^\infty$. 

\begin{lemma}
    \label{lemma:AVIconvergence}
    Let Assumptions~\ref{assumption:model}~and~\ref{assumption:restrictedDiscount} hold, and let $\{j_i\}_{i=0}^\infty \subset \mathbb{Z}_{\geq 0}$ be a non-decreasing set of iteration indices, with $j_0 = 0$ and $j_{i+1} \geq j_i$.  Fix some index $i \in \mathbb{Z}_{\geq 0}$, and, for all indices $\ell \leq i$, define $\varepsilon_\ell > 0$ such that $V'_{j_\ell}(s) - V_{j_{\ell+1}}(s) \leq \varepsilon_\ell$ for all $s \in \mathcal{S}'$.  Then, for all $s \in \mathcal{S}'$,
    \begin{equation*}
        V^*(s) - V_{j_{i+1}}(s) \leq \gamma^{i+1} + \sum_{\ell=0}^i \gamma^{i-\ell} \varepsilon_\ell.
    \end{equation*}
\end{lemma}

\begin{proposition}
    \label{proposition:convergencetoAVI}
    Let Assumptions~\ref{assumption:model}~and~\ref{assumption:restrictedDiscount} hold, and let $\{j_i\}_{i=0}^\infty \subset \mathbb{Z}_{\geq 0}$ be a non-decreasing set of iteration indices, with $j_0 = 0$ and $j_{i+1} \geq j_i$.  For all $s \in \mathcal{S}'$ and $i \in \mathbb{Z}_{\geq 0}$,
    \begin{equation*}
        V'_{j_i}(s) - V_{j_{i+1}}(s) \leq 1 - \left (\underline{\pi} \prod_{\ell = 0}^{i-1} \left (\frac{Q_{j_\ell}(s,a^*_{j_\ell}(s))}{Q^*(s,a^*_{j_\ell}(s))} \right )^{(j_{\ell + 1} - j_{\ell})} \right )^{\frac{1}{1 + j_{i+1}}}.
    \end{equation*}
\end{proposition}

\begin{proof}
    Fix some $i \in \mathbb{Z}_{\geq 0}$ and $s \in \mathcal{S}'$.  We have that
    \begin{equation}
        \label{eq:ProductConvergenceToAVI}
        \begin{aligned}
            \left(\frac{V_{j_{i+1}}(s)}{V'_{j_i}(s)}  \right )^{1 + j_{i+1}} & \geq \prod_{j=0}^{j_{i+1}} \frac{V_j(s)}{V'_{j_i}(s)} \geq \pi_0 (a^*_{j_i}(s) \cond s) \prod_{j=0}^{j_{i+1}}\frac{Q_j(s,a^*_{j_i}(s))}{Q_{j_i}(s,a^*_{j_i}(s))}\\
            &\geq \underline{\pi} \left (\frac{Q_{j_i}(s,a^*_{j_i}(s))}{Q_{j_i}(s,a^*_{j_i}(s))} \right )^{(1 + j_{i+1} - j_i)} \prod_{\ell=0}^{i-1} \left(\frac{Q_{j_\ell}(s,a^*_{j_i}(s))}{Q_{j_i}(s,a^*_{j_i}(s))} \right)^{(j_{\ell+1}-j_\ell)},
        \end{aligned}
    \end{equation}
    where the first inequality follows from the monotonicity of the value functions (Lemma~\ref{lemma:monotonic}), the second by selecting a single action in the closed-form representation~\eqref{eq:SCClosedForm} and from the definition of $a^*_j(s)$, and the third from the monotonicity of the Q-functions (Lemma~\ref{lemma:monotonic}). Then,
    \begin{equation*}
        V'_{j_i}(s) - V_{j_{i+1}}(s) = V'_{j_i}(s) \left ( 1 - \frac{V_{j_{i+1}}(s)}{V'_{j_i}(s)} \right ) \leq 1 - \left (\underline{\pi}\prod_{\ell=0}^{i-1} \left(\frac{Q_{j_\ell}(s,a^*_{j_i}(s))}{Q_{j_i}(s,a^*_{j_i}(s))} \right)^{(j_{\ell+1}-j_\ell)} \right )^{\frac{1}{1+j_{i+1}}},
    \end{equation*}
    where the final inequality follows from Assumption~\ref{assumption:restrictedDiscount} and taking the $(1+j_{i+1})$-th root of~\eqref{eq:ProductConvergenceToAVI}.  This relies on the observation that the lower bound in~\eqref{eq:ProductConvergenceToAVI} is no larger than $1$.
\end{proof}

Theorem~\ref{theorem:discountedConvergenceRate} derives a convergence rate for discounted MDPs by combining the previous two results with appropriately selected indices $j_i$ and errors $\varepsilon_i$.  Specifically, we show that a policy with an optimality gap of $\varepsilon$ is obtained within $\mathcal{O}(1/\varepsilon^p)$ iterations, where $p$ is determined by problem data and some tunable parameter $B$.  In Theorem~\ref{theorem:discountedConvergenceRate}, $p>1$ as the denominator in its definition~\eqref{eq:discountedConvergenceParameters} is negative.  

To prove Theorem~\ref{theorem:discountedConvergenceRate}, we leverage the convergence of the SC value functions $V_{j_{i+1}}$ to the AVI iterates $V'_{j_i}$, as described in Proposition~\ref{proposition:convergencetoAVI}.  We first define parameters $B$, $k$, $j_i$, $\varepsilon_i$, and $\Delta_i$. Parameters $j_i$ separate the SC iterations into blocks $\{j_i+1,\dots,j_{i+1}\}$, where each block corresponds to an AVI step.  Parameter $\varepsilon_i$ gives an upper bound on the error between the SC iterate $V_{j_{i+1}}$ and the AVI iterate $V'_{j_i}$, and $\Delta_i$ is an upper bound on the error between the SC iterate $V_{j_i}$ and the optimal value function $V^*$.  Parameter $B$ represents the proportion of the overall error $\Delta_i$ allocated to the AVI error $\varepsilon_i$, that is, $\varepsilon_i = B \Delta_i$.  Finally, $k$ denotes the number of AVI steps required to ensure an overall error of $\Delta_k \leq \varepsilon$. Using Proposition~\ref{proposition:convergencetoAVI}, we show that the SC-to-AVI errors are bounded by $\varepsilon_i$ at each iterate $j_{i}$.  We then use Lemma~\ref{lemma:AVIconvergence} to show that the overall errors are bounded by $\Delta_i$.

Lemma~\ref{lemma:logRatioBound} gives a relevant bound on the ratio of logarithms, which is used in the proof of Theorem~\ref{theorem:discountedConvergenceRate}.  In Lemma~\ref{lemma:logRatioBound} and Theorem~\ref{theorem:discountedConvergenceRate}, we take $\log(1/(1-x))/x = 1$ at $x =0$, so that the expression is continuous on $[0,1)$.

\begin{lemma}
    \label{lemma:logRatioBound}
    Let $a \in (0,1)$, $b \in \mathbb{R}_{>0}$, and $c \in [0,1)$.  Then, 
    \begin{equation*}
        \frac{\log(\max\{1-b,\,1-c\})}{\log(1-a)} \leq \frac{b\log(1/(1-c))}{ac}.
    \end{equation*}
\end{lemma}

\begin{theorem}
    \label{theorem:discountedConvergenceRate}
    Let Assumptions~\ref{assumption:model}~and~\ref{assumption:restrictedDiscount} hold and fix some $\varepsilon \in (0,1)$.  The following result holds for any $B \in (0,1-\gamma)$.  Define
    \begin{equation}
        \label{eq:discountedConvergenceParameters}
        C := \max \left \{\frac{\gamma \log(1/\Phi)}{B \Phi (1-\Phi)},\, 1 \right \} \quad \text{and} \quad p := 1 - \frac{\log(2C)}{\log (\gamma + B)}.
    \end{equation}
    Then, for all iterations $j \in \mathbb{Z}_{\geq 0}$ that satisfy
    \begin{equation}
        \label{eq:iterationBound}
        j \geq \left ( \frac{\log(1/\underline{\pi})}{B} \right ) \left (\frac{1}{\varepsilon} \right )^p,
    \end{equation}
    it holds that $V^*(s) - V_j(s) \leq \varepsilon$ for all $s \in \mathcal{S}'$.
\end{theorem}

\begin{proof}
    Fix some $\varepsilon \in (0,1)$ and $B \in (0,1-\gamma)$.  For $i \in \mathbb{Z}_{\geq 0}$, define $\Delta_i := (\gamma + B)^i$ and $\varepsilon_i := B \Delta_i$.
    We also define iteration indices
    \begin{equation*}
        k := \ceil{\frac{\log(\varepsilon)}{\log(\gamma + B)}} \quad \text{and} \quad j_{i+1} := \floor{\left ( \frac{\log(\underline{\pi})}{\log(1 - \varepsilon_i)} \right ) (2C)^i} \quad \forall i \in \mathbb{Z}_{\geq 0},
    \end{equation*}
    with $j_0 = 0$.  From the definitions, we establish two relations between these quantities: for all $i \in \mathbb{Z}_{\geq 0}$, it holds that
    \begin{equation}
        \label{eq:EpsilonIterRelation}
        \varepsilon_i \geq 1 - \left (\underline{\pi} \prod_{\ell=0}^{i-1} \left ( \max \left \{\Phi,\, 1 - \frac{\gamma \Delta_\ell}{\Phi} \right \} \right )^{(j_{\ell+1} - j_\ell)} \right )^{\frac{1}{1+j_{i+1}}}
    \end{equation}
    and
    \begin{equation}
        \label{eq:DeltaEpsilonRelation}
        \Delta_i = \gamma^i + \sum_{\ell=0}^{i-1} \gamma^{i-\ell-1} \varepsilon_\ell.
    \end{equation}

    To show~\eqref{eq:EpsilonIterRelation}, we have that
    \begin{equation*}
        \begin{aligned}
            j_{i+1} + 1 & \geq \left (\frac{\log(\underline{\pi})}{\log(1 - \varepsilon_i)} \right ) (2 C )^i\\
            & = \left (\frac{\log(\underline{\pi})}{\log(1 - \varepsilon_i)} \right ) \left (1 + \sum_{\ell=0}^{i-1} 2^\ell \right ) C^i\\
            & \geq \left (\frac{\log(\underline{\pi})}{\log(1 - \varepsilon_i)} \right ) \left (1 + \sum_{\ell=0}^{i-1} \frac{\gamma \log(1/\Phi)}{B \Phi (1-\Phi)} (2C)^\ell\right )\\
            & \geq \left (\frac{\log(\underline{\pi})}{\log(1 - \varepsilon_i)} \right ) \left (1 + \sum_{\ell=0}^{i-1} \frac{\log(\max\{\Phi,\, 1-\gamma \Delta_\ell / \Phi \})}{\log(1 - \varepsilon_\ell)} (2C)^\ell \right )\\
            & \geq \left (\frac{1}{\log(1 - \varepsilon_i)} \right ) \left (\log(\underline{\pi}) + \sum_{\ell=0}^{i-1} \log(\max \{\Phi,\, 1 - \gamma \Delta_\ell / \Phi \})j_{\ell+1} \right )\\
            & \geq \left (\frac{1}{\log(1 - \varepsilon_i)} \right ) \left (\log(\underline{\pi}) + \sum_{\ell=0}^{i-1} \log(\max \{\Phi,\, 1 - \gamma \Delta_\ell / \Phi \})(j_{\ell+1} - j_\ell) \right ),
        \end{aligned}
    \end{equation*}
    where the first relation follows by definition of $j_{i+1}$, the second from the identity $2^n -1= \sum_{i=0}^{n-1} 2^i$, the third from $C^i \geq C^{\ell+1}$ for $\ell \leq i-1$, the fourth from Lemma~\ref{lemma:logRatioBound} with $(a,b,c) = (\varepsilon_\ell, \gamma \Delta_\ell/\Phi,1-\Phi)$, and the fifth by definition of $j_{\ell+1}$.  Rearranging the final inequality gives that
    \begin{equation*}
        \log(1 - \varepsilon_i) \leq \frac{1}{1 + j_{i+1}} \left ( \log(\underline{\pi}) + \sum_{\ell=0}^{i-1} \log(\max \{\Phi,\, 1 - \gamma \Delta_\ell / \Phi \})(j_{\ell+1} - j_\ell)\right),
    \end{equation*}
    and exponentiation of both sides yields~\eqref{eq:EpsilonIterRelation}.
    
    To show~\eqref{eq:DeltaEpsilonRelation}, observe as a base case that $\Delta_0 = 1 = \gamma^0$.  Suppose for induction that~\eqref{eq:DeltaEpsilonRelation} holds at index $i \in \mathbb{Z}_{\geq 0}$.  Then,
    \begin{equation*}
        \Delta_{i+1} = (\gamma + B) \Delta_i = \varepsilon_i + \gamma^{i+1} + \sum_{\ell=0}^{i-1} \gamma^{i-\ell} \varepsilon_\ell = \gamma^{i+1} + \sum_{\ell=0}^{i} \gamma^{i-\ell} \varepsilon_\ell,
    \end{equation*}
    and~\eqref{eq:DeltaEpsilonRelation} holds at index $i+1$.  By induction, we conclude that~\eqref{eq:DeltaEpsilonRelation} holds for all $i \in \mathbb{Z}_{\geq 0}$.

    We now proceed to prove the result.  Suppose for induction that, at index $i \in \mathbb{Z}_{\geq 0}$, the following condition holds for all $s \in \mathcal{S}'$ and all $\ell \in \mathbb{Z}_{\geq 0}$ such that $\ell \leq i-1$:
    \begin{equation}
        \label{eq:GapInductiveHypothesis}
        V'_{j_{\ell}}(s) - V_{j_{\ell+1}}(s) \leq \varepsilon_{\ell}.
    \end{equation}
    As a base case, this condition is trivially satisfied when $i=0$.  By Lemma~\ref{lemma:AVIconvergence}, we have for all $\ell \leq i$ and $s \in \mathcal{S}'$ that 
    \begin{equation}
        \label{eq:absoluteErrorDelta}
        V^*(s) - V_{j_{\ell}}(s) \leq \gamma^{\ell} + \sum_{m=0}^{\ell-1} \gamma^{\ell-m-1} \varepsilon_m = \Delta_{\ell},
    \end{equation}
    where the final equality is due to~\eqref{eq:DeltaEpsilonRelation}.  Then, for all $\ell \leq i$ and $s \in \mathcal{S}'$, we have
    \begin{equation*}
        \begin{aligned}
            \frac{Q_{j_\ell}(s,a^*_{j_\ell}(s))}{Q^*(s,a^*_{j_\ell}(s))} & = 1 - \left ( \frac{Q^*(s,a^*_{j_\ell}(s)) - Q_{j_\ell}(s,a^*_{j_\ell}(s))}{Q^*(s,a^*_{j_\ell}(s))} \right )\\
            & \geq 1 - \left ( \frac{\gamma \sum_{s' \in \mathcal{S}} P(s' \cond s,a^*_{j_\ell}(s)) \left( V^*(s') - V_{j_\ell}(s') \right)}{\Phi}\right ) \geq 1 - \frac{\gamma \Delta_\ell}{\Phi},
        \end{aligned}
    \end{equation*}
    as $\Phi \leq Q_0(s,a^*_0(s)) \leq Q_{j_\ell}(s,a^*_{j_\ell}(s)) \leq Q^*(s,a^*_{j_\ell}(s))$ by the monotonicity of the Q-functions (Lemma~\ref{lemma:monotonic}).
    We also have under Assumption~\ref{assumption:restrictedDiscount} that
    \begin{equation*}
        \frac{Q_{j_\ell}(s,a^*_{j_\ell}(s))}{Q^*(s,a^*_{j_\ell}(s))} \geq \Phi,
    \end{equation*}
    as $Q^*(s,a) \leq 1$.  Therefore, by Proposition~\ref{proposition:convergencetoAVI}, we have for all $s \in \mathcal{S}'$ that
    \begin{equation*}
        \begin{aligned}
            V'_{j_i}(s) - V_{j_{i+1}}(s) & \leq 1 - \left ( \underline{\pi} \prod_{\ell=0}^{i-1} \left (\frac{Q_{j_\ell}(s,a^*_{j_\ell}(s))}{Q^*(s,a^*_{j_\ell}(s))} \right ) ^{(j_{\ell+1}-j_\ell)} \right ) ^{\frac{1}{1 + j_{i+1}}}\\
            & \leq 1 - \left ( \underline{\pi} \prod_{\ell=0}^{i-1} \left (\max \left \{\Phi,\, 1 - \frac{\gamma \Delta_\ell}{\Phi} \right \}\right ) ^{(j_{\ell+1}-j_\ell)} \right ) ^{\frac{1}{1 + j_{i+1}}} \leq \varepsilon_i,
        \end{aligned}
    \end{equation*}
    where the final inequality follows from~\eqref{eq:EpsilonIterRelation}.  The inductive hypothesis~\eqref{eq:GapInductiveHypothesis} is then satisfied at index $i+1$. Then, both~\eqref{eq:GapInductiveHypothesis}~and~\eqref{eq:absoluteErrorDelta} hold for any $i \in \mathbb{Z}_{\geq 0}$.  
    
    From~\eqref{eq:absoluteErrorDelta}, for all $s \in \mathcal{S}'$,
    \begin{equation*}
        V^*(s) - V_{j_k}(s) \leq (\gamma + B)^k \leq \varepsilon.
    \end{equation*}
    Therefore, at the $j_k$-th iteration, the optimality gap is no larger than $\varepsilon$.  For $j$ satisfying~\eqref{eq:iterationBound}, 
    \begin{equation*}
        \begin{aligned}
            j & \geq \left (\frac{\log(1/\underline{\pi})}{B} \right ) \left (\frac{1}{\varepsilon}\right) (\varepsilon)^{\frac{\log(2C)}{\log(\gamma + B)}} = \left (\frac{\log(1/\underline{\pi})}{B} \right ) (\gamma+B)^{\frac{-\log(\varepsilon)}{\log(\gamma + B)}}(2C )^{\frac{\log(\varepsilon)}{\log(\gamma + B)}}\\
            & \geq \left (\frac{\log(1/\underline{\pi})}{B} \right ) \left ( \frac{2C}{\gamma + B} \right ) ^{k-1} = \left (\frac{\log(1/\underline{\pi})}{\varepsilon_{k-1}} \right ) (2C)^{k-1}\geq \left ( \frac{\log(\underline{\pi})}{\log(1-\varepsilon_{k-1})} \right ) (2C)^{k-1} \geq j_k,
        \end{aligned}
    \end{equation*}
    where we apply the property $\log(1-x) \leq -x$ for $x \in (-\infty,1)$ in the third inequality.  By Lemma~\ref{lemma:monotonic}, we then have $V^*(s) - V_j(s) \leq \varepsilon$, as $j \geq j_k$.
\end{proof}

Corollary~\ref{corollary:exactRate} gives an explicit, parameter-free convergence rate that follows from Theorem~\ref{theorem:discountedConvergenceRate} by taking $B = \min \{\gamma/\Phi^2,\ (1-\gamma)/2\}$.  The resulting rate falls into two regimes that differ in the exponent $p$.  A proof of this result is provided in Appendix~\ref{app:proofs}.

\begin{corollary}
    \label{corollary:exactRate}
    Let Assumptions~\ref{assumption:model}~and~\ref{assumption:restrictedDiscount} hold and fix some $\varepsilon \in (0,1)$.  Define 
    \begin{gather*}
    j^* := \begin{cases}
             \left (\frac{2 \log(1/\underline{\pi})}{1 - \gamma} \right ) \left (\frac{1}{\varepsilon} \right )^{p_1} \quad & \text{if}\ \  2 \gamma \geq  \Phi^2 (1-\gamma),\\
            \left (\frac{\log(1/\underline{\pi}) \Phi^2}{\gamma} \right ) \left ( \frac{1}{\varepsilon} \right )^{p_2} \quad & \text{otherwise},
            \end{cases}\\
        \text{where} \quad p_1 := 1 - \frac{\log \left ( \frac{4 \gamma}{\Phi^2 (1-\gamma)}\right )}{\log((1+\gamma)/2)} \quad \text{and} \quad p_2 := 1 - \frac{\log(2)}{\log(\gamma+\gamma/\Phi^2)}.
    \end{gather*}
    Then, for $j \in \mathbb{Z}_{\geq 0}$ such that $j \geq j^*$, it holds that $V^*(s) - V_j(s) \leq \varepsilon$ for all $s \in \mathcal{S}'$.
\end{corollary}

Corollary~\ref{corollary:exactRate} illustrates the behavior of the convergence threshold $j^*$ as the values of $\gamma$ and $\Phi$ vary.  When $\gamma$ is large ($\sim 1$) and $\Phi$ is small ($\sim 0$), the convergence rate is $\mathcal{O}(1/\varepsilon^{p_1})$.  For small $\gamma$ and large $\Phi$, the convergence rate is $\mathcal{O}(1/\varepsilon^{p_2})$.  The exponents $p_1$ and $p_2$ increase as $\gamma$ increases and decrease as $\Phi$ increases.  Thus, the order $p$ of the convergence rate $\mathcal{O}(1/\varepsilon^p)$ is small when the problem is highly discounted ($\gamma \sim 0$) and when the best action under the initial policy yields large returns ($\Phi \sim 1$).

The single-period MDP described in Section~\ref{sec:singlePeriod} is equivalent to a maximally-discounted MDP, with $\gamma = 0$.  In Corollary~\ref{corollary:exactRate}, as $\gamma \rightarrow 0$, the second regime applies and $p_2 \rightarrow 1$.  For this high-discount setting, the limiting behavior of Corollary~\ref{corollary:exactRate} recovers an asymptotic convergence rate of $\mathcal{O}(1/\varepsilon)$, which is slower than the $\mathcal{O}(\log(1/\varepsilon))$ rate of Theorem~\ref{theorem:MABConvergenceRate}.  This suggests that there are key differences between single-period and general MDPs, or that Theorem~\ref{theorem:discountedConvergenceRate} gives a loose rate as $\gamma$ approaches zero.

\section{Discussion and Conclusion}

Theorem~\ref{theorem:limitConvergence} establishes that SC converges to a restricted optimal solution for well-behaved MDPs satisfying Assumption~\ref{assumption:model}.  Theorem~\ref{theorem:MABConvergenceRate} shows that SC converges to an $\varepsilon$-optimal policy within $\mathcal{O}(\log(1/\varepsilon))$ iterations for single-period MDPs, and Theorem~\ref{theorem:discountedConvergenceRate} shows that SC converges within $\mathcal{O}(1/\varepsilon^p)$ iterations for discounted MDPs, with an exponent $p \geq 1$ that depends strongly on problem data and may be arbitrarily large.  These results address several open questions on the convergence behavior of SC.  Our convergence rates give upper bounds on the worst-case performance of SC; SC may converge faster on typical instances, and further analysis may prove convergence at a faster rate.

Relative to comparable policy optimization algorithms, the SC convergence rates presented in this work are slow.  For single-period MDPs, selecting the action with the largest reward recovers an exact optimal policy in a single iteration.  For discounted MDP, value iteration and policy iteration both converge linearly, yielding an $\varepsilon$-optimal policy in $\mathcal{O}(\log(1/\varepsilon))$ iterations \citep{puterman2014markov}.  Policy iteration may even converge quadratically, requiring only $\mathcal{O}(\log(\log(1/\varepsilon)))$ iterations \citep[Thm.~6.4.8]{puterman2014markov}. 

The slow convergence of SC can be viewed as a feature.  \citet{russo2026success} describes SC as a ``conservative'' operator that does not deviate significantly from the incumbent policy $\pi_0$.  Changes in the policy under the SC update directly correspond to improvements in its performance \citep[][Prop.~4.4]{russo2026success}.  Although this conservatism inhibits efficient convergence, it permits stability in learning by avoiding drastic changes in the distribution of actions.  SC also exhibits stability by converging to an optimal \emph{restricted} solution, and cannot produce policies that take actions which were not taken by the initial policy $\pi_0$.

Another strength of SC is that it generates stochastic policies, in which actions that were supported under the initial policy remain supported at every iteration.  In contrast, value and policy iteration generate policies that deterministically select a single action.  Stochastic policies are desirable for many applications, especially in the context of generative AI \citep{fan2018hierarchical,wu2025role}.  Additionally, stochastic policies permit on-policy learning, where the Q-functions $Q_k(s,a)$ are approximated by sampling trajectories under the incumbent policy $\pi_k$.  As SC maintains non-zero probabilities on the set of supported actions, a sufficiently large number of trajectory samples will visit all supported state-action pairs, permitting an estimate of the Q-functions.  Because deterministic policies always select a single action, 
value and policy iteration are not compatible with Q-function estimation by on-policy sampling.  

Similar properties (conservatism and stochastic policy iterates) are shared by trust-region policy optimization \citep[TRPO;][]{schulman2015trust} and proximal policy optimization \citep[PPO;][]{schulman2017proximal}.  However, these methods require parameter tuning and generally do not have well-understood convergence rates, although some variants converge in $\mathcal{O}(1/\varepsilon^2)$ iterations \citep{shani2020adaptive}.  For comparison to TRPO and PPO, interesting open directions include designing on-policy variants of SC that estimate the Q-functions from trajectory samples and investigating the sample complexity of these variants.

\section*{Acknowledgements}
We thank Daniel Russo for introducing us to this topic and providing his perspectives on open directions and challenges. We also are grateful to the anonymous referees whose comments suggested a path to tighten the single-period convergence rates.

{
\small
\bibliographystyle{unsrtnat}
\bibliography{SCConvergence}
}

\newpage
\appendix

\section{Comparison to Reward-Weighted Regression}
\label{app:RLRcompare}

Reward-weighted regression (RWR) is an algorithm that applies expectation-maximization to improve policies of MDPs.  RWR generates policy iterates by the rule
\begin{equation*}
    \pi_{k+1} ( a \cond s) \in \underset{\pi \in \Pi}{\argmax} \ \ \mathbb{E}_{\pi_k} \left[ Q_{\pi_k}(s,a) \log(\pi(a \cond s))\right ].
\end{equation*}
\citet{strupl2022reward} prove two key results on RWR for compact state and action spaces, under the assumption that all rewards are strictly positive ($R(s,a) > 0$) and all state-action pairs are supported in the initial policy ($\pi(a \cond s) > 0$).  First, the RWR update is equivalent to the SC update~\eqref{eq:successConditioning} \citep[][Thm.~3.1]{strupl2022reward}.  Second, RWR converges in the limit to an optimal policy \citep[][Thm.~5.1]{strupl2022reward}.  Together, these results demonstrate that both SC and RWR converge to an optimal solution.  This convergence property corresponds to our Theorem~\ref{theorem:limitConvergence}, under a modified set of assumptions.

\citet{strupl2022reward} also analyze the asymptotic convergence rate of RWR, showing that the optimality gap is bounded by a geometric sequence after sufficiently many iterations.  Theorem~\ref{theorem:asymptoticRate} presents a version of this result.  We take $\mathcal{A}^*_0(s) := \{a \in \mathcal{A}_0(s) \,:\, Q^*(s,a) = V^*(s)\}$ as the set of optimal supported actions at state $s \in \mathcal{S}'$ and define the following parameters:
\begin{equation*}
    \lambda := \max_{s \in \mathcal{S}'} \ \max_{a \in \mathcal{A}_0(s) \setminus A^*_0(s)} \frac{Q^*(s,a)}{V^*(s)} \quad \text{and} \quad \alpha := \frac{2 \lambda}{0.9 + 1.1\lambda}.
\end{equation*}
The parameter $\lambda$ is the ratio between the Q-value of the second-best action and of the best action, and $\alpha < 1$ is the common ratio of the bounding geometric sequence.  

\begin{theorem}[\citep{strupl2022reward}, Corr.~5.1]
\label{theorem:asymptoticRate}
    Let Assumption~\ref{assumption:restrictedDiscount} hold, and suppose that $R(s,a) > 0$ for all $s \in \mathcal{S}$ and $a \in \mathcal{A}_0(s)$.  Let $n \in \mathbb{Z}_{\geq 0}$ such that $Q^*(s,a) - Q_{n}(s,a) < \frac{1 - \lambda}{20}$ for all $s \in \mathcal{S}'$ and $a \in \mathcal{A}_0(s)$.  Then, there exists some constant $F > 0$ such that $$V^*(s) - V_j(s) \leq F \cdot \alpha^{j-n}$$
    for all $j \in \mathbb{Z}_{\geq 0}$ with $j > n$ and all $s \in \mathcal{S}'$.
\end{theorem}

In Theorem~\ref{theorem:asymptoticRate}, the optimality gap begins to converge linearly at the $n$-th iteration, where $n$ is chosen as an iteration where the gap between the incumbent and optimal Q-functions is sufficiently small.  This approach does not characterize how large $n$ must be to ensure the condition is met; the coefficient $F$ may also grow arbitrarily large with $n$.  In contrast, Theorems~\ref{theorem:MABConvergenceRate}~and~\ref{theorem:discountedConvergenceRate} give explicit convergence rates which quantify the number of iterations required to guarantee a bound on the optimality gap.

The setting of Theorem~\ref{theorem:asymptoticRate} has similarities to the single-period MDP described in Section~\ref{sec:singlePeriod}.  First, the condition $Q^*(s,a) - Q_n(s,a) < \frac{1-\lambda}{20}$ ensures that the Q-functions correctly indicate which action is optimal; that is, there is no $a \not \in \mathcal{A}^*_0(s)$ such that $Q_n(s,a) > Q_n(s,a^*)$, where $a^* \in \mathcal{A}^*_0(s)$.  Similarly, for single-period MDPs, the Q-functions $Q_j(s^\circ,a)$ are exactly the rewards $R(s^\circ,a)$ and therefore indicate the set of optimal actions at any iteration $j$.  Second, the geometric rate in Theorem~\ref{theorem:asymptoticRate} depends on the ratio $\lambda$ between the values of the first- and second-best actions.  The proof of Theorem~\ref{theorem:MABConvergenceRate} similarly gives that $V^*(s^\circ) - V_k(s^\circ) \leq \frac{1-D}{DE^k}$, where the common ratio $1/E$ is equivalent to $\lambda$.

\section{Deferred Proofs}
\label{app:proofs}

\subsection*{Proof of Lemma~\ref{lemma:monotonic}}

\begin{proof}
    Let $B_{\pi}$ be the Bellman operator under policy $\pi$, defined for $V : \mathcal{S} \rightarrow \mathbb{R}$ as
    \begin{equation*}
        (B_{\pi} V)(s) = \sum_{a \in \mathcal{A}} \pi(a \cond s) \left (R(s,a) + \gamma \sum_{s' \in \mathcal{S}} P(s' \cond s,a) V(s') \right ).
    \end{equation*}
    This operator preserves monotonicity; that is, $V_{+} \succeq V_{-}$ implies that $B_{\pi} V_{+} \succeq B_{\pi} V_{-}$, where $\succeq$ is the component-wise inequality over states $s \in \mathcal{S}$.  Under Assumption~\ref{assumption:model}, a fixed point of $B_{\pi}$ is the unique solution $V_\pi$ to the Bellman equations~\eqref{eq:bellmanValue}.

    Let $B_k := B_{\pi_k}$.  Consider the sequence $\{B^{(j)}_{k+1} V_{k}\}_{j=0}^\infty$, where the superscript $(j)$ denotes that $B_{k+1}$ is applied $j$ times to $V_{k}$.  Observe that, for $s\in \mathcal{S}'$,
    \begin{equation*}
        \begin{aligned}
            (B_{k+1} V_{k}) (s) & = \sum_{a \in \mathcal{A}} \pi_{k+1}(a \cond s) Q_{k}(s,a) = \sum_{a \in \mathcal{A}} \frac{\pi_k(a \cond s) Q_{k}(s,a)^2}{V_{k}(s)}\\
            & \geq \frac{1}{V_{k}(s)}\left ( \sum_{a \in \mathcal{A}} \pi_k(a \cond s) Q_{k}(s,a) \right )^2 = V_{k}(s),
        \end{aligned}
    \end{equation*}
    where the inequality is due to Jensen's inequality.  Thus, $B_{k+1} V_{k} \succeq V_{k}$.  By the preservation of monotonicity, $B_{k+1}^{(j+1)} V_{k} \succeq B_{k+1}^{(j)} V_{k} \succeq V_k$ for any $j \in \mathbb{Z}_{\geq 0}$.
    
    Under Assumption~\ref{assumption:model}, there exist some finite optimal value functions $\overline{V}$ and Q-functions $\overline{Q}$ such that
    \begin{equation*}
        \overline{V}(s) = \max_{a \in \mathcal{A}} \ \overline{Q}(s,a) \quad \text{and} \quad \overline{Q}(s,a) = R(s,a) + \gamma \sum_{s' \in \mathcal{S}} P(s' \cond s,a) \overline{V}(s').
    \end{equation*}
    Then,
    \begin{equation*}
        (B_{k+1} \overline{V})(s) = \sum_{a \in \mathcal{A}} \pi_{k+1}(a \cond s) \overline{Q}(s,a) \leq \max_{a \in \mathcal{A}} \ \overline{Q}(s,a) = \overline{V}(s),
    \end{equation*}
    and $\overline{V} \succeq B_{k+1}\overline{V}$.  Again by the preservation of monotonicity, $\overline{V} \succeq B_{k+1}^{(j)}\overline{V}$ for all $j \in \mathbb{Z}_{\geq 0}$.  As $\overline{V} \succeq V_k$, we have $B^{(j)}_{k+1} \overline{V} \succeq B^{(j)}_{k+1} V_k$.  Thus, $\overline{V} \succeq B_{k+1}^{(j)} V_k$ and the sequence $\{B^{(j)}_{k+1} V_{k}\}_{j=0}^\infty$ is bounded.

    Therefore, $\{B^{(j)}_{k+1} V_{k}\}_{j=0}^\infty$ is a bounded non-decreasing sequence and has a limit. By the continuity of the operator $B_{k+1}$, we have
    \begin{equation*}
        L := \lim_{j \rightarrow \infty} \ B^{(j)}_{k+1} V_{k} = B_{k+1}\lim_{j \rightarrow \infty} \ B^{(j-1)}_{k+1} V_{k} = B_{k+1} L.
    \end{equation*}
    Thus, $L$ is a fixed point of $B_{k+1}$, and therefore $L = V_{k+1}$. By the monotonicity of the sequence, we conclude that $V_{k+1} = L \succeq V_k$.  By~\eqref{eq:valueQRelation}, $V_{k+1} \succeq V_k$ directly implies that $Q_{k+1}(s,a) \geq Q_{k}(s,a)$ for all $s \in \mathcal{S}'$ and $a \in \mathcal{A}$.
\end{proof}

\subsection*{Proof of Lemma~\ref{lemma:AVIconvergence}}

\begin{proof}
    The proof follows standard analysis for the contractivity of Bellman operators.  For all $\ell \in \mathbb{Z}_{\geq 0}$ such that $\ell \leq i$,
    \begin{equation*}
        \begin{aligned}
            \max_{s \in \mathcal{S}}\ \left ( V^*(s) - V_{j_{\ell+1}}(s) \right ) & \leq \varepsilon_\ell + \max_{s \in \mathcal{S}}\ \left ( V^*(s) - V'_{j_\ell}(s) \right )\\
            & = \varepsilon_\ell + \max_{s \in \mathcal{S}}\ \left (\max_{a \in \mathcal{A}_0(s)}\ Q^*(s,a) - \max_{a \in \mathcal{A}_0(s)}\ Q_{j_\ell}(s,a) \right )\\
            & \leq \varepsilon_\ell + \max_{s \in \mathcal{S}}\ \max_{a \in \mathcal{A}_0(s)}\ \left (Q^*(s,a) - Q_{j_\ell}(s,a) \right )\\
            & \leq \varepsilon_\ell + \gamma\max_{s' \in \mathcal{S}}\ \left ( V^*(s') - V_{j_\ell}(s') \right ).
        \end{aligned}
    \end{equation*}
    Nesting the above inequalities over indices $\ell$ gives that
    \begin{equation*}
        \max_{s \in \mathcal{S}}\ \left (V^*(s) - V_{j_{i+1}}(s) \right ) \leq \gamma^{i+1} \max_{s \in \mathcal{S}}\ \left ( V^*(s) - V_0(s) \right ) + \sum_{\ell=0}^i \gamma^{i-\ell} \varepsilon_\ell.
    \end{equation*}
    By Assumptions~\ref{assumption:model}~and~\ref{assumption:restrictedDiscount}, $V^*(s) - V_0(s) \leq 1$ for all $s \in \mathcal{S}$, which gives the desired result.
\end{proof}

\subsection*{Proof of Lemma~\ref{lemma:logRatioBound}}

\begin{proof}
    If $c = 0$, we have 
    \begin{equation*}
        \frac{\log(\max\{1-b,\,1-c\})}{\log(1-a)} = 0 < \frac{b}{a} = \frac{b \log(1/(1-c))}{ac},
    \end{equation*}
    where we take $\log(1/(1-x))/x = 1$ at $x = 0$.
    
    Otherwise, $c > 0$.  Let $b' := \min\{b,\, c\} \in (0,1)$.  As the function $\log(1-x)$ is concave, it is upper bounded by its tangent at $x=0$, so $\log(1-a) \leq -a$.  As $\log(1-b') < 0$, this gives that
    \begin{equation*}
        \frac{\log(1-b')}{\log(1-a)} \leq \frac{-\log(1-b')}{a}.
    \end{equation*}
    Next, we write $b'$ as a convex combination of $c$ and $0$:
    \begin{equation*}
        b' = \left ( \frac{b'}{c} \right ) c + \left ( \frac{c-b'}{c} \right ) 0,
    \end{equation*}
    where the coefficients fall on $[0,1]$ by definition of $b'$.  By Jensen's inequality,
    \begin{equation*}
        \log(1-b') \geq \left ( \frac{b'}{c} \right ) \log(1-c) + \left ( \frac{c-b'}{c} \right ) \log(1) = \frac{b'\log(1-c)}{c}.
    \end{equation*}
    Therefore, we conclude that
    \begin{equation*}
        \frac{\log(1-b')}{\log(1-a)} \leq \frac{-b' \log(1-c)}{ac} \leq \frac{b\log(1/(1-c))}{ac}.
    \end{equation*}
\end{proof}

\subsection*{Proof of Corollary~\ref{corollary:exactRate}}

\begin{proof}
    Fix some $\varepsilon \in (0,1)$.  As $\Phi > 0$,
    \begin{equation}
        \label{eq:negLogRatio}
        \frac{\log(1/\Phi)}{1-\Phi} \leq \frac{1/\Phi - 1}{1-\Phi} = \frac{1}{\Phi},
    \end{equation}
    where the inequality follows from $\log(x) \leq x - 1$ for all $x \in \mathbb{R}_{> 0}$.  Let
    \begin{equation*}
        B := \min \left \{\frac{\gamma}{\Phi^2},\, \frac{1-\gamma}{2} \right \}.
    \end{equation*}
    By definition, $B \in (0,1-\gamma)$.  Let $C$ and $p$ be defined with respect to $B$ as in~\eqref{eq:discountedConvergenceParameters}, and consider any $j \in \mathbb{Z}_{\geq 0}$ such that $j \geq j^*$.

    First, consider the case where $2\gamma \geq \Phi^2(1-\gamma)$.  In this case, $B = (1-\gamma)/2$ and
    \begin{equation*}
        C = \max \left \{\frac{\gamma \log(1/\Phi)}{B \Phi (1-\Phi)},\, 1 \right \} \leq \max \left \{\frac{2\gamma}{\Phi^2(1-\gamma)},\, 1  \right \} = \frac{2\gamma}{\Phi^2(1-\gamma)},
    \end{equation*}
    where the inequality is due to~\eqref{eq:negLogRatio}.  Then, 
    \begin{equation*}
        p = 1 - \frac{\log(2C)}{\log((1+\gamma)/2)} \leq p_1,
    \end{equation*}
    as $\log((1+\gamma)/2) < 0$ by Assumption~\ref{assumption:restrictedDiscount}.  Thus,
    \begin{equation*}
        j \geq j^* = \left (\frac{2 \log(1/\underline{\pi})}{1 - \gamma} \right ) \left (\frac{1}{\varepsilon} \right )^{p_1} \geq \left (\frac{\log(1/\underline{\pi})}{B} \right ) \left (\frac{1}{\varepsilon} \right )^{p},
    \end{equation*}
    and $V^*(s) - V_j(s) \leq \varepsilon$ for all $s \in \mathcal{S}'$ by Theorem~\ref{theorem:discountedConvergenceRate}.

    Next, consider the other case, in which $2 \gamma < \Phi^2 (1-\gamma)$ and $B = \gamma / \Phi^2$.  Now,
    \begin{equation*}
        1 \leq C =  \max \left \{\frac{\gamma \log(1/\Phi)}{B \Phi (1-\Phi)},\, 1 \right \} \leq \max \left \{\frac{\gamma}{B \Phi^2},\, 1  \right \} = 1,
    \end{equation*}
    so $C = 1$.  The inequality is again due to~\eqref{eq:negLogRatio}.  Then,
    \begin{equation*}
        p = 1 - \frac{\log(2)}{\log(\gamma + \gamma/\Phi^2)} = p_2.
    \end{equation*}
    Thus,
    \begin{equation*}
        j \geq j^* = \left (\frac{\log(1/\underline{\pi}) \Phi^2}{\gamma} \right ) \left ( \frac{1}{\varepsilon} \right )^{p_2} = \left (\frac{\log(1/\underline{\pi})}{B} \right ) \left (\frac{1}{\varepsilon} \right )^{p},
    \end{equation*}
    and again $V^*(s) - V_j(s) \leq \varepsilon$ for all $s \in \mathcal{S}'$ by Theorem~\ref{theorem:discountedConvergenceRate}.
\end{proof}

\end{document}